\documentclass{article}
\usepackage[utf8]{inputenc}

\usepackage[colorlinks,linkcolor=magenta,citecolor=blue, pagebackref=true]{hyperref}
\usepackage[margin=1.2in]{geometry}
\usepackage[numbers]{natbib}
\usepackage{booktabs}
\usepackage{amsmath}
\usepackage{amsfonts}
\usepackage{amsthm}
\usepackage{graphicx}

\newtheorem{theorem}{Theorem}
\newtheorem{lemma}{Lemma}

\newtheorem{definition}{Definition}
\usepackage{thm-restate}
\usepackage[capitalise]{cleveref}
\usepackage{comment}
\DeclareMathOperator{\faxen}{Fi}
\DeclareMathOperator{\polylog}{Li}

\newtheorem{manualtheoreminner}{Theorem}
\newenvironment{manualtheorem}[1]{%
  \renewcommand\themanualtheoreminner{#1}%
  \manualtheoreminner
}{\endmanualtheoreminner}

\newtheorem{manualpropositionminner}{Proposition}
\newenvironment{manualproposition}[1]{%
  \renewcommand\themanualpropositionminner{#1}%
  \manualpropositionminner
}{\endmanualpropositionminner}

\newenvironment{naligned}{
    \begin{equation}
    \begin{aligned}
}{
    \end{aligned}
    \end{equation}
    \ignorespacesafterend
}

\newcommand{\codeurl}{\url{https://github.com/microsoft/csrobust}}

\title{A lower confidence sequence for the changing mean of non-negative right heavy-tailed observations with bounded mean}
\author{Paul Mineiro}
\begin{document}

\maketitle

\begin{abstract}
A confidence sequence (CS) is an anytime-valid sequential inference primitive which produces an adapted sequence of sets for a predictable parameter sequence with a time-uniform coverage guarantee.  This work constructs a non-parametric non-asymptotic lower CS for the running average conditional expectation whose slack converges to zero given non-negative right heavy-tailed observations with bounded mean.  Specifically, when the variance is finite the approach dominates the empirical Bernstein supermartingale of \citet{howard2021time}; with infinite variance, can adapt to a known or unknown $(1 + \delta)$-th moment bound; and can be efficiently approximated using a sublinear number of sufficient statistics.  In certain cases this lower CS can be converted into a closed-interval CS whose width converges to zero, e.g., any bounded realization, or post contextual-bandit inference with bounded rewards and unbounded importance weights.  A reference implementation and example simulations demonstrate the technique.

\end{abstract}

\section{Introduction}
Recently the A/B testing and contextual bandit communities have embraced anytime-valid strategies to facilitate composition of arbitrary decision logic into online experimental procedures.\citep{johari2022always,karampatziakis2021off}  A confidence sequence (CS) is an anytime-valid sequential inference primitive which, for any $\alpha \in (0, 1)$, produces an adapted sequence $\text{CI}_t$ of sets for a predictable parameter sequence of interest $\theta_t$ with the guarantee $\mathrm{Pr}\left(\forall t \geq 1: \theta_t \in \text{CI}_t\right) \geq 1 - \alpha$: non-parametric non-asymptotic variants have broad utility.

This work is a robust lower CS for the average conditional expectation of an adapted sequence of non-negative right heavy-tailed observations.  The basic approach assumes a non-negative scalar observation with bounded mean and produces a lower CS for the running mean, i.e., a CS of the form $\text{CI}_t = \left[L_t, \infty\right)$.  A lower CS has immediate utility for pessimistic decision scenarios such as gated deployment, i.e.,  requiring changes to production environments to certify improvement with high probability.\cite{karampatziakis2021off} Furthermore, even with unbounded observations, this lower CS can sometimes be converted into a closed-interval CS, e.g., post contextual-bandit inference with bounded rewards and unbounded importance weights.\citep{waudby2022} Bounded observations always permit constructing a CS: this is sensible despite all moments being finite, as the proposed approach is both theoretically and empirically superior to the empirical Bernstein supermartingale of \citet{howard2021time}.

\subsection{Contributions}

In \cref{sec:runningmean}, we introduce a novel test supermartingale for the running mean. We prove the following: 
the supermartingale dominates the empirical Bernstein supermartingale of \citet{howard2021time}; does not require finite variance to converge, but can instead adapt to a known or unknown $(1 + \delta)$-th moment bound; and can be efficiently approximated using a sublinear number of sufficient statistics.  The result is the imminently practical doubly-discrete robust mixture of \cref{thm:ddrm}.  We provide simulations in \cref{sec:simulations}; code to reproduce all results is available at \codeurl.

\section{Related Work}

Anytime-valid sequential inference is an active research area with a rich history dating back to \citet{wald1945sequential}.  Here we only discuss aspects relevant to the current work, and we refer the interested reader to the excellent overview contained in \citet{waudby2020estimating}.  

\citet{waudby2020estimating} derive time-uniform confidence sequences for a fixed parameter and bounded observations.  Their constructions can produce a lower CS when observations are unbounded above with bounded mean.  Unfortunately their techniques are only applicable to a fixed parameter: when the parameter is changing, their techniques cover a data-dependent weighted mixture, providing limited utility.  The fixed parameter case is not restricted to stationary processes, e.g., sampling without replacement is governed by a fixed parameter despite the data distributions predictably changing.  Nonetheless, the fixed parameter restriction limits the applicability, e.g., when the conditional mean is changing over time.

\citet{wang2022catoni} improve upon \citet{waudby2020estimating} by leveraging Catoni-style influence functions, including an infinite variance result for a known $(1 + \delta)$-th moment bound.  Adapting to an unknown bound is provably impossible in their setting due to \citet{bahadur1956nonexistence}, whereas our more restrictive assumption of bounded mean allows us to adapt.  The approach shares the limitations of \citet{waudby2020estimating} with respect to a changing parameter.

\citet{howard2021time} describes multiple non-asymptotic boundaries for the running average conditional expectation.  In particular they propose the empirical Bernstein boundary, which has zero asymptotic slack for lower-bounded right heavy-tailed observations with bounded mean and finite variance, e.g., log-normally distributed.  However, all of their one-sided discrete time boundaries require finite variance for asymptotically zero slack.

\section{Derivations}
\label{sec:runningmean}

The following novel construction is a test martingale qua \citet{shafer2011test}.
\begin{restatable}[Heavy NSM]{definition}{runningmeannsm}
\label{def:runningmeannsm}
Let $\left(\Omega, \mathcal{F}, \{ \mathcal{F}_t \}_{t \in \mathbb{N}}, P\right)$ be a filtered probability space, and let $\{ X_t \}_{t \in \mathbb{N}}$ be an adapted non-negative $\mathbb{R}$-valued random process with bounded mean $\mathbb{E}_t\left[X_t\right] \leq 1$.  Let $\{ \hat{X}_t \}_{t \in \mathbb{N}}$ be a predictable $[0, 1]$-valued random process, and let $\lambda \in [0, 1)$ be a constant bet.  Then
\begin{equation}
E_t(\lambda) \doteq \exp\left(\lambda \left(\sum_{s \leq t} \hat{X}_s - \sum_{s \leq t} \mathbb{E}_{s-1}\left[X_s\right] \right) \right) \prod_{s \leq t} \left(1 + \lambda \left(X_s - \hat{X}_s\right) \right).
\label{eqn:runningmeannsm}    
\end{equation}
\end{restatable}
\noindent \cref{def:runningmeannsm} is a non-negative supermartingale: letting $Y_t \doteq X_t - \mathbb{E}_{t-1}\left[X_t\right]$ and $\delta_t \doteq \hat{X}_t - \mathbb{E}_{t-1}\left[X_t\right]$, $$
\mathbb{E}_{t-1}\left[\frac{E_t(\lambda)}{E_{t-1}(\lambda)}\right] = \mathbb{E}_{t-1}\left[ \exp\left(\lambda_t \delta_t\right) \left(1 + \lambda_t \left(Y_t - \delta_t\right)\right) \right] \stackrel{(a)}{=} \exp\left(\lambda_t \delta_t\right) \left(1 - \lambda_t \delta_t\right) \stackrel{(b)}{\leq} 1,
$$ where $(a)$ is due to $(\mathbb{E}_{t-1}[Y_t] = 0 \text{ and } \delta_t \text{ predictable})$, and $(b)$ is due to $e^{x} (1 - x) \leq 1$.

\subsection{Statistical Considerations}

\paragraph{Empirical Bernstein Dominance} \cref{def:runningmeannsm} is essentially the empirical Bernstein supermartingale of \citet[Section A.8]{howard2021time} without the slack from the \citet{fan2015exponential} inequality.  Consequently it is straightforward to show dominance.

\begin{definition}[Empirical Bernstein Supermartingale]
\begin{equation}
\label{eqn:empbern}
B_t(\lambda) = \exp\left(\lambda \sum_{s \leq t} Y_s - \psi_E(\lambda) \sum_{s \leq t} \left(X_s - \hat{X}_s\right)^2\right),
\end{equation}
where $\psi_E(x) \doteq -x - \log(1 - x)$.
\end{definition}

\begin{theorem} For the same $\lambda$ and $\hat{X}$, \cref{eqn:runningmeannsm} is at least as wealthy as \cref{eqn:empbern}.
\begin{proof}
The log wealth is $$
\begin{aligned}
\log\left(E_t(\lambda)\right) %
&= \lambda \sum_{s \leq t} Y_s - \sum_{s \leq t} \left( \lambda \left(X_s - \hat{X}_s\right) - \log\left(1 + \lambda \left(X_s - \hat{X}_s\right) \right) \right) \\
&= \lambda \sum_{s \leq t} Y_s - \psi_E(\lambda) \sum_{s \leq t} \frac{\psi_E\left(-\lambda \left(X_s - \hat{X}_s\right)\right)}{\psi_E(\lambda)} & \left( \psi_E(x) \doteq -x - \log\left(1 - x\right) \right) \\
&\geq \lambda \sum_{s \leq t} Y_s - \psi_E(\lambda) \sum_{s \leq t} \left(X_s - \hat{X}_s\right)^2 & \left(\text{\citet{fan2015exponential}}\right). \\
\end{aligned}
$$ 
\end{proof}
\end{theorem}
\noindent Thus \cref{eqn:runningmeannsm} inherits the asymptotic guarantee of the mixed empirical Bernstein martingale.

\begin{definition}[Mixture boundary]
\label{def:mixboundary}
For any probability distribution $F$ on $\left[0, \lambda_{\max}\right)$ and $\alpha \in [0, 1]$, $$
\mathcal{M}_{\alpha}\left(X_{s \leq t}, \hat{X}_{s \leq t}\right) \doteq \sup\left\{ y \in \mathbb{R} : \int_0^{\lambda_{\max}} E_t(\lambda) dF(\lambda) \leq \frac{1}{\alpha} \right\},
$$ is a time-uniform crossing boundary with probability at most $\alpha$ for $Y_t = \sum_{s \leq t} \left(X_s - \mathbb{E}_{s-1}\left[ X_s \right]\right)$.
\end{definition}

\begin{manualproposition}{Howard 2}[\citet{howard2021time}]
\label{prop:finitevariance}
If $F$ is absolutely continuous wrt Lebesque measure in a neighborhood around zero, the mixture boundary is upper bounded by $\sqrt{v \log\left(\frac{v}{2 \pi \alpha^2 f^2(0)}\right) + o(1)}$ as $v \to \infty$, where $v = \sum_{s \leq t} \left(X_s - \hat{X}_s\right)^2$, and $f(0) = \frac{dF}{d\lambda}(0)$.
\end{manualproposition}
\noindent Computationally this is less felicitous as closed-form conjugate mixtures are not available for \cref{eqn:runningmeannsm}. We revisit computational issues later in this section.

\paragraph{Heavy Tailed Results} When the conditional second moment is not bounded, \cref{prop:finitevariance} provides no guarantee because the variance process grows superlinearly.  However, unlike the empirical Bernstein process, \cref{def:runningmeannsm} can induce confidence sequences that shrink to zero asymptotically even if the conditional second moment is unbounded, and can adapt to an unknown $(1 + \delta)$-th moment bound.  This is essentially because the function $\left(x - \log(1 + x)\right)$ asymptotically grows more slowly than $x^q$ for any $q > 1$.  
\begin{restatable}[$q$-growth]{lemma}{qgrowth}
\label{lemma:qgrowth}
For any $q \in (1, 2]$ and $\lambda \in \left[0, 1 + W_0(-e^{-2})\right] \approx [0, 0.841]$ where $W_0(z)$ is the principal branch of the Lambert W function, $$
\lambda x - \log\left(1 + \lambda x\right) \leq \lambda^q \left( 1_{x \leq 0} x^2 + 1_{x > 0} \min\left(x^2, c^*(q) x^q\right) \right),
$$ where for $q < 2$, $c^*(q) \doteq x^*(q)^{2-q}$ where $x^*(q) > 0$ uniquely solves $$
q = \frac{x^*(q)^2}{(1 + x^*(q)) (x^*(q) - \log(1 + x^*(q)))},
$$ and $\lim_{q \uparrow 2} c^*(q) = 1$ defines $c^*(2)$.
\end{restatable}
\begin{proof}
See \cref{app:infinitevariance}.
\end{proof}
\noindent For example when $q = \frac{3}{2}$, $c^*(q) \approx 1.35$.  Combining the $q$-growth lemma with a modified version of Laplace's method yields \cref{thm:qasym}.
\begin{restatable}[$q$-asymptotics]{theorem}{qasym}
\label{thm:qasym}
For the mixture boundary of \cref{def:mixboundary}, for any $q \in (1, 2]$, any $\lambda_{\max} \in (0, 1 + W_0\left(-e^{-2}\right] \approx (0, 0.841]$, and any $F$ absolutely continous with Lebseque measure in a neighborhood of zero with $\frac{dF}{d\lambda}(\lambda) = f(0) \lambda^{q/2 - 1} + O(\lambda^{q/2})$ and $f(0) > 0$, the mixture boundary is at most $$
\begin{aligned}
\mathcal{M}_{\alpha}\left(X_{\leq t}, \hat{X}_{\leq t}\right) &\leq v^{1/q} \left( \log\left(\frac{\sqrt{v}}{\alpha f(0) a(q) \left(1 + o(1)\right)}\right) \right)^{\frac{q-1}{q}},
\end{aligned}
$$ where $$
\begin{aligned}
v &\doteq \sum_{s \leq t} 1_{X_s \leq \hat{X}_s} \left(X_s - \hat{X}_s\right)^2 + 1_{X_s > \hat{X}_s} \min\left(\left(X_s - \hat{X}_s\right)^2, c^*(q) \left(X_s - \hat{X}_s\right)^q\right), \\
a(q) &\doteq \sqrt{\frac{2 \pi q}{4 (q - 1) v}} \exp\left(\frac{1}{q}^{\frac{1}{q-1}} - \frac{1}{q}^{\frac{q}{q-1}}\right), \\
\end{aligned}
$$ with $c^*(q)$ as in \cref{lemma:qgrowth}.
\end{restatable}
\begin{proof}
See \cref{app:infinitevariance}.
\end{proof}
\noindent \cref{thm:qasym} gives a rate of $O\left(v^{1/q} (\log v\right)^{(q-1)/q})$: for comparison the $q^\mathrm{th}$-moment law of the iterated logarithm is $O\left(v^{1/q} (\log \log v\right)^{(q-1)/q})$.\cite{shao1997self}  Thus, like the finite variance case, the mixture method achieves the LIL rate to within a logarithmic factor.  Note the quantity $v$ appearing in \cref{thm:qasym} is for analysis only and need not be explicitly computed; rather \cref{def:mixboundary} is used.  However \cref{thm:qasym} cannot directly adapt to an unknown moment bound, as it requires a specification of the moment being bounded in order to construct the mixture distribution with the appropriate integrable singularity at the origin.  Given \cref{lemma:qgrowth} it is reasonable to seek adaptation to an unknown moment bound\footnote{Note the impossibility result of \citet{bahadur1956nonexistence} does not apply because the mean is bounded.} which we achieve via a discrete mixture over \cref{thm:qasym}.

\begin{restatable}[$q$-adaptive]{corollary}{qadapt}
\label{cor:qadapt}
For $\lambda_{\max} \in (0, 1 + W_0\left(-e^{-2}\right)] \approx (0, 0.841]$, let $$
\begin{aligned}
F(\lambda) &= \sum_{k=0}^\infty w_k \frac{q(k)}{ 2 \lambda_{\max}^{q(k)/2}} \lambda^{q(k)/2 - 1},
\end{aligned}
$$ where $q(k) = 1 + \eta^k$, $\eta \in (0, 1)$, and $1 = \sum_{k=0}^\infty w_k$.  Then for any $q \in (1, 2]$, the mixture of \cref{def:mixboundary} guarantees $$
\begin{aligned}
\mathcal{M}_{\alpha}\left(X_{\leq t}, \hat{X}_{\leq t}\right) &\leq v^{1/\tilde{q}} \left( \log\left(\frac{\sqrt{v}}{\alpha w_{k(q)} f(0) a(\tilde{q}) \left(1 + o(1)\right)}\right) \right)^{\frac{\tilde{q}-1}{\tilde{q}}},
\end{aligned}
$$ where $$
\begin{aligned}
v &\doteq \sum_{s \leq t} 1_{X_s \leq \hat{X}_s} \left(X_s - \hat{X}_s\right)^2 + 1_{X_s > \hat{X}_s} \min\left(\left(X_s - \hat{X}_s\right)^2, c^*(\tilde{q}) \left(X_s - \hat{X}_s\right)^{\tilde{q}}\right), \\
k(q) &\doteq \lceil \log_{\eta}\left(q - 1\right) \rceil, \\
\tilde{q} &\doteq 1 + \eta^{k(q)} %
= q - (q - 1) \left(1 - \eta^{\Delta(q)}\right) \geq q - (q - 1) (1 - \eta), \\
\Delta(q) &\doteq \lceil \log_{\eta}\left(q - 1\right) \rceil - \log_{\eta}\left(q - 1\right) \in [0, 1),
\end{aligned}
$$ with $a(q)$ as in \cref{thm:qasym} and $c^*(q)$ as in \cref{lemma:qgrowth}.
\end{restatable}
\noindent \cref{cor:qadapt} is conservative: it ignores the contribution to the mixture wealth from all but one component.  Nonetheless,   choosing $w_k = \zeta(r)^{-1} (k + 1)^{-r}$ with $r > 1$
for \cref{cor:qadapt} yields a $\log \log_{\eta}(\tilde{q} - 1)$ degradation in the boundary while adapting to a $\tilde{q}$ which is at most $(1 - \eta)$ closer to 1 than $q$.

\subsection{Computational Considerations}

Computationally, \cref{eqn:runningmeannsm} is less convenient than \cref{eqn:empbern} in several respects:
\begin{enumerate}
    \item Computing \cref{eqn:runningmeannsm} for different $\lambda$ naively requires keeping around all of history, which is $O(t)$ space and time.  By comparison for any $\lambda$, \cref{eqn:empbern} can be computed constant space and time using the sufficient statistic $\sum_{s \leq t} \left(X_s - \hat{X}_s\right)^2$.
    \item The mixture boundary of \cref{def:mixboundary} must be computed numerically for \cref{eqn:runningmeannsm}, whereas \cref{eqn:empbern} has a closed-form conjugate mixture.
\end{enumerate}
We mitigate the first issue by exploiting strong convexity of the variance process, and we mitigate the second issue by using discrete mixtures with early termination.

\paragraph{Approximate Sufficient Statistics} Define $g(\lambda, x) \doteq \lambda x - \log\left(1 + \lambda x\right)$ and write \cref{eqn:runningmeannsm} as $$
\begin{aligned}
\log\left(E_t(\lambda)\right) &= \lambda \sum_{s \leq t} Y_s - \sum_{s \leq t} g\left(\lambda, X_s - \hat{X}_s\right). 
\end{aligned}
$$ If we can upper bound $g(\lambda, x)$, we can lower bound $\log\left(E_t(\lambda)\right)$ and therefore upper bound any resulting boundary.  We use the following upper bound, 
\begin{naligned}
\label{eqn:tildeg}
\tilde{g}(\lambda, x; k) &= \begin{cases} 
\frac{\lambda^2}{2 (1 - \lambda)} x^2 & x \in (-1, 0] \\
\frac{\lambda^2}{2} x^2 & x \in (0, 1] \\
\alpha g(\lambda, x_1) + (1 - \alpha) g(\lambda, x_2) - \frac{1}{2} m \alpha (1 - \alpha) (x_2 - x_1)^2 & x > 1
\end{cases}, 
\end{naligned}
where for $x > 1$, $$
\begin{aligned}
x &\in [x_1, x_2) = [k^n, k^{n+1}) \\
\alpha &\doteq \frac{k^{n+1} - x}{k^{n+1} - k^n}, \\
m &\doteq \frac{\lambda^2}{\left(1 + \lambda x_2\right)^2}, \\
\end{aligned}
$$ with $k > 1$ setting the resolution of our exponential grid of sufficient statistics.

\begin{restatable}[$\tilde g$-approximation]{theorem}{tildegprox}
\label{thm:tildegprox}
For all $\lambda \in [0, 1)$, $x > -1$, and $k > 1$, $$
\begin{aligned}
g(\lambda, x) \leq \tilde{g}(\lambda, x; k) &= g(\lambda, x) + \begin{cases}
O\left(\lambda^3\right) & x \leq 1 \\
O\left(\left(k - 1\right)^3 \lambda^3\right) & x > 1 
\end{cases}.
\end{aligned}
$$ Furthermore $\tilde{g}$ can be computed in $\min\left(3 t, 2 + 3 \log_k\left(\left| X_{s \leq t}\right|_{\infty}\right)\right)$ space and time.
\end{restatable}
\begin{proof}
See \cref{app:computationalconsiderations}.
\end{proof}

\paragraph{Discrete Mixture} Given $\tilde{g}(\lambda, x; k)$, the mixture boundary of \cref{def:mixboundary} for a single $q$ can be computed using numerical quadrature.  In practice this is wasteful because a generic quadrature routine will spend compute on refinements that do not change the decision boundary.  By contrast the discrete mixtures of \citet{howard2021time} can be early terminated.  Combining this with the discrete mixture of \cref{cor:qadapt} and lower-bounding the inner sum leads to the doubly-discrete robust mixture.
\begin{restatable}[DDRM]{theorem}{ddrm}
\label{thm:ddrm}
For fixed $\lambda_{\max} \in (0, 1 + W_0(-e^{-2}) \approx 0.841]$, $\xi > 1$, $r > 1$, $k > 1$ and $\eta \in (0, 1)$, 
$$
\begin{aligned}
\mathcal{DM}_{\alpha}\left(X_{s \leq t}, \hat{X}_{s \leq t}\right) &= \sup\left\{ y \in \mathbb{R} : \sum_{j=0}^\infty z_j \exp\left(\lambda_j y - \sum_{s \leq t} \tilde{g}\left(\lambda_j, X_s - \hat{X}_s; k\right) \right) \leq \frac{1}{\alpha} \right\}, \\
\lambda_j &\doteq \frac{\lambda_{\max}}{\xi^{j+1/2}}, \\
z_j &\doteq \frac{1}{2} \left(\frac{\xi - 1}{\xi^{1 + j}}\right) \left(1 + \frac{1 + \polylog_r(\eta)}{\eta \zeta(r)}\right),
\end{aligned}
$$ where $\zeta(r)$ is the Riemann zeta function and $\polylog_r(\eta)$ is the Jonqui\`ere polylogarithm,
is a time-uniform crossing boundary with probability at most $\alpha$ for $Y_t = \sum_{s \leq t} \left(X_s - \mathbb{E}_{s-1}\left[ X_s \right]\right)$.
\end{restatable}
\begin{proof}
See \cref{app:computationalconsiderations}.
\end{proof}
\noindent Note \cref{thm:ddrm} only asserts coverage, and not the other properties of \cref{cor:qadapt}.  Hopefully, approximating \cref{cor:qadapt} preserves the beneficial properties.  In practice we terminate the sum of \cref{thm:ddrm} dynamically, either because the wealth is above the threshold, or because we can upper bound the wealth below the threshold.

\section{Simulations}
\label{sec:simulations}

For demonstration we simulate contextual bandit off-policy estimation.\footnote{Code to reproduce all results available at \codeurl.} Using the framework of \citet{waudby2022} we transform a lower CS on a non-negative random variable with mean upper bounded by 1 into a lower and upper CS for a contextual bandit problem with bounded rewards.  Specifically, the raw observations are tuples of importance-weighted rewards $(W_t, R_t)$ where a.s. $W_t \geq 0$, $\mathbb{E}_t[W_t]=1$, and a.s. $R_t \in [0, 1]$; the lower CS is obtained via $X_t = W_t R_t$; and the upper CS is obtained from one minus the lower CS on $X_t = W_t (1 - R_t)$.  The importance weight and reward distributions can be predictably adversarially chosen, corresponding to any combination of changing environment (e.g., reward per action is changing), changing logging policy (e.g., the logging policy is from an online contextual bandit learning algorithm), and changing policy being evaluated (e.g., the evaluated policy is from an offline contextual bandit learning algorithm).  We simulate an empirical Bernstein CS with gamma conjugate prior and $\rho = 1$; and the DDRM of \cref{thm:ddrm} with $\lambda_{\max} = 1/2$, $\xi = 8/5$, $r = 2$, $k = 3/2$, and $\eta = 0.95$.

\begin{figure}
\centering
\begin{minipage}[t]{.49\textwidth}
  \vskip 0pt
  \centering
  \includegraphics[width=.96\linewidth]{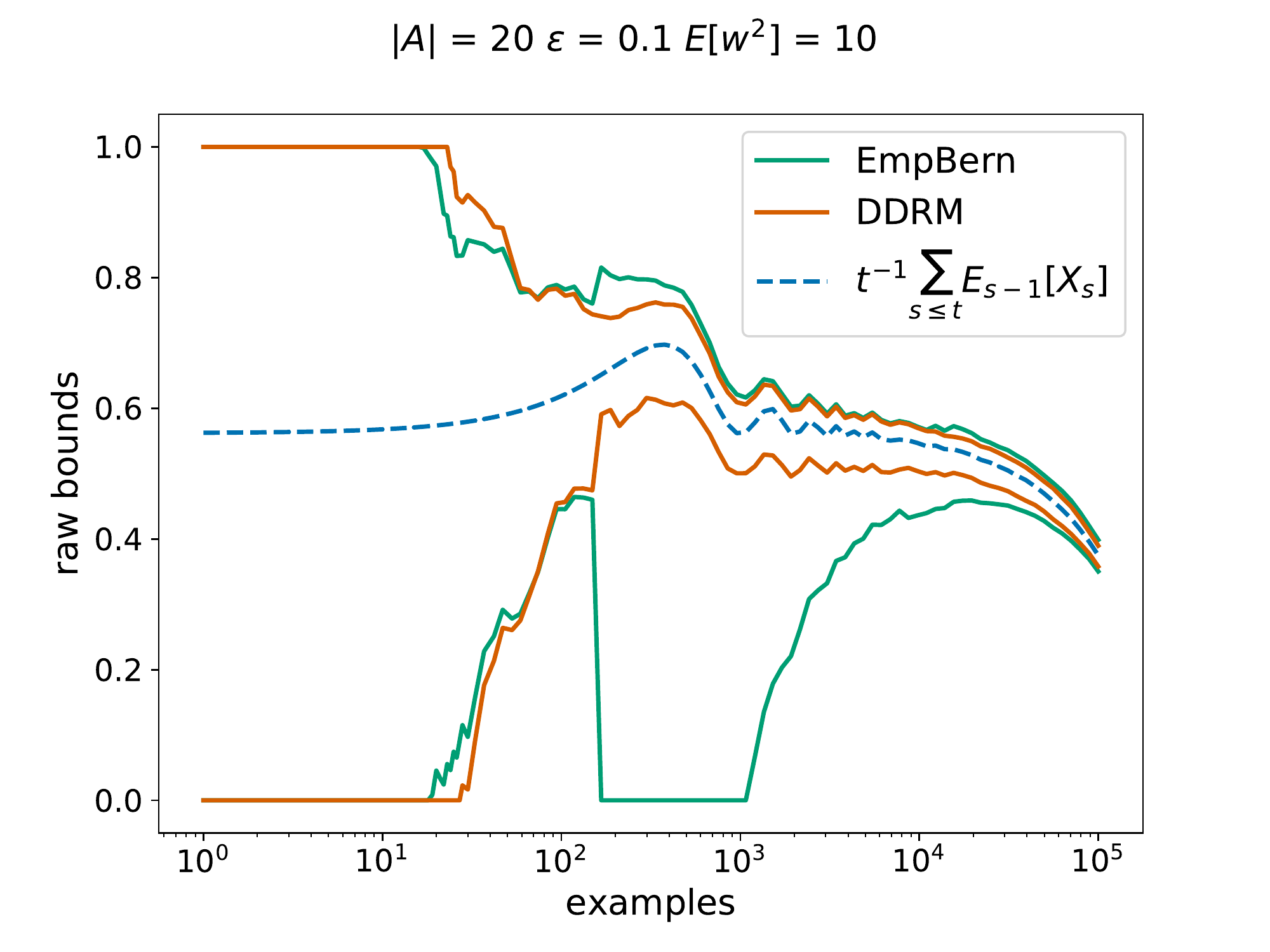}
  \vskip -12pt
  \caption{Epsilon-greedy CB.}
  \label{fig:epsgreedy}
\end{minipage}
\hfill
\begin{minipage}[t]{.49\textwidth}
  \vskip 0pt
  \centering
  \includegraphics[width=.96\linewidth]{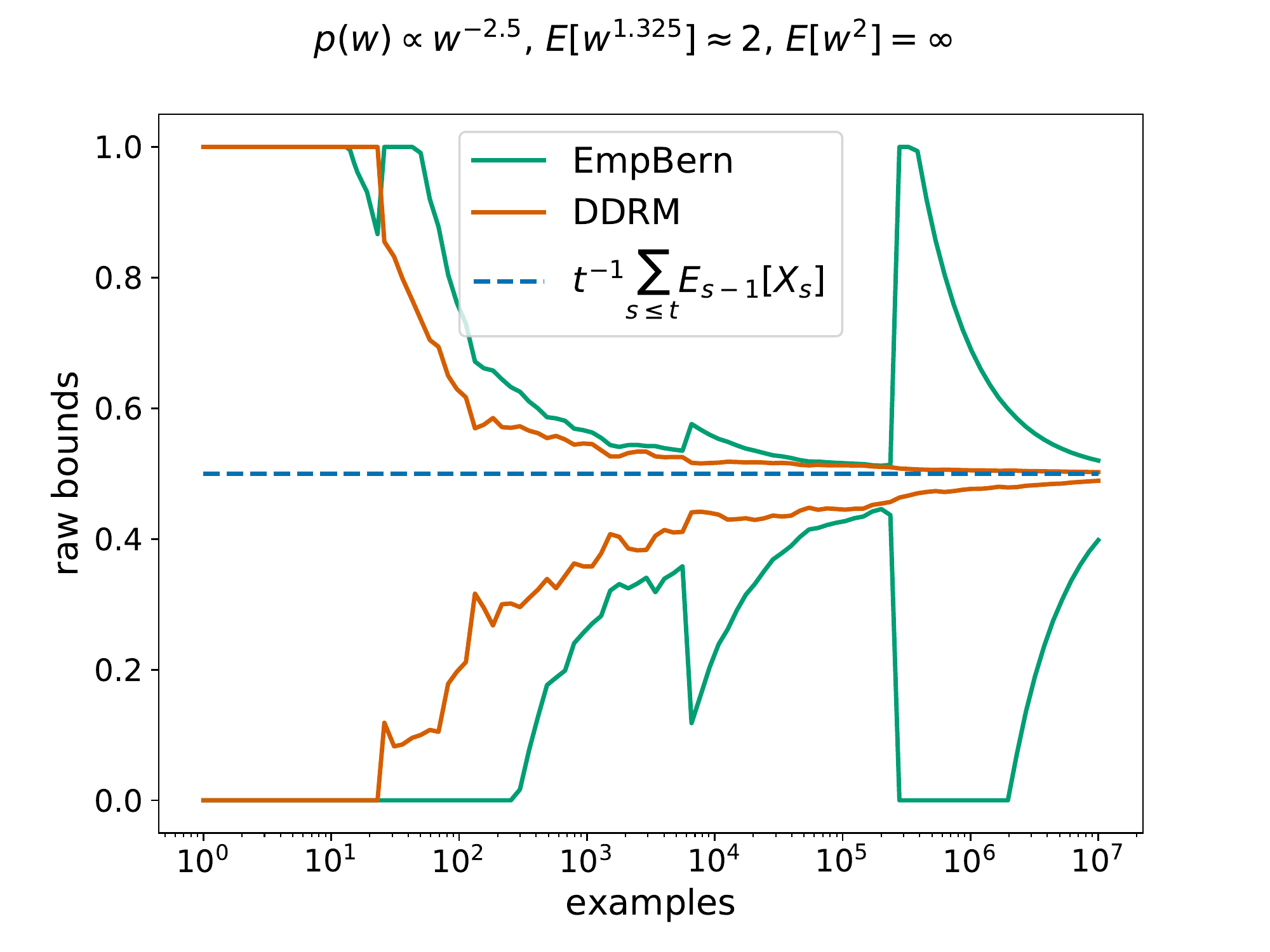}
  \vskip -12pt
  \caption{Continuous action CB.}
  \label{fig:infinitevar}
\end{minipage}
\end{figure}

\cref{fig:epsgreedy} simulates off-policy estimation in an environment where rewards exhibit diurnal seasonality and a negative long-term trend, and the logging policy is $\epsilon$-greedy over a discrete action set.  The number of actions (20) and amount of exploration (0.1) are typical production settings.  The policy being evaluated frequently agrees with the logging policy and hence has importance-weight variance $E[w^2] = 10$ much lower than the worst case of 190, which is typical of policies produced by off-policy learning algorithms.  The occurrence of the largest importance weight at circa 100 examples causes a visible downward adjustment in the lower CS for empirical Bernstein, but DDRM is unaffected.  Thus, even when the variance is finite and the range is bounded, DDRM exhibits beneficial robustness to outliers.  If the lower CS is used for gated deployment of policy improvements as in \citet{karampatziakis2021off}, this increased robustness implies improvements can be reliably deployed more rapidly.  Computationally, for this simulation on the author's laptop, empirical Bernstein can compute a CS over $10^6$ points in circa 0.3ms, whereas DDRM takes circa 20ms: slower but acceptable.

\cref{fig:infinitevar} simulates off-policy estimation in a constant environment and Pareto-distributed importance weights with infinite variance.  This simulation is inspired by off-policy estimation in continuous action spaces.  The empirical Bernstein confidence sequence has no asymptotic width guarantees as the variance process increases superlinearly.  A large realization at circa $10^5$ examples illustrates the issue, as the empirical Bernstein CS is completely reset: if we continue running the simulation forever in infinite precision arithmetic, an infinite number of these complete resets will occur.  By contrast, the DDRM is able to adapt to the (unknown) finite lower moment and converge.

\section{Discussion}

This paper advocates \cref{thm:ddrm} as an always-preferable drop-in replacement for the empirical Bernstein conjugate mixture of \citet{howard2021time}.  Statistically, the underlying martingale and associated mixtures are superior; computationally, approximations mitigate the increased cost.

The structure of \cref{def:runningmeannsm} is intriguing: it contains a betting martingale term, but instead of betting that a fixed parameter is close to the actuals, it bets that the predictions are close to the actuals.  This centers the martingale at the summed conditional mean plus the prediction error: in contrast, betting martingales are either centered at a weighted sum of the conditional means (if a fixed parameter is used), or define feasible sets on the Cartesian product of the parameter space (if a different parameter is used per timestep).

By testing if the running mean of an identification statistic is zero, \cref{cor:qadapt} can be modified to extend the framework of \citet{casgrain2022anytime} to in-hindsight elicitation, i.e., elicitable hypotheses of the running average conditional distribution  $t^{-1} \sum_{s \leq t} \mathrm{Pr}\left(\cdot | \mathcal{F}_s\right)$. For example, expectiles of the average historical distribution are in-hindsight elicitable, but the identification statistic can exhibit infinite variance, e.g., a Pareto distributed observation.  If the mean is bounded, \cref{cor:qadapt} can adapt to an unknown moment bound of the identification statistic.  We will elaborate this in future work.

\bibliographystyle{plainnat}
\bibliography{references.bib}

\newpage
\appendix

\section{Heavy Tailed Results}
\label{app:infinitevariance}

\subsection{Proof of \cref{lemma:qgrowth}}

\qgrowth*
\begin{proof}
The lemma is trivially true for $q = 2$ via \citet{fan2015exponential}.  For $q < 2$, consider
$$
\begin{aligned}
\gamma(x; q) &\doteq \frac{x - \log\left(1 + x\right)}{x^q}, \\
\frac{\partial}{\partial x} \gamma(x; q) &= \frac{x^{1-q}}{1 + x} \left( q (1 + x) \left( x - \log\left(1 + x\right) \right) - x^2 \right).
\end{aligned}
$$ This last expression is positive whenever $$
\begin{aligned}
q &\geq \frac{x^2}{(1 + x) \left(\log\left(1 + x\right) - x \right)} \doteq \beta(x).
\end{aligned}
$$ Noting $\beta(x)$ is strictly decreasing for $x > 0$, this implies $\gamma(x; q)$ is strictly increasing whenever $x \geq x^*(q)$.  Further noting $\lim_{x \to 0} \beta(x) \to 2$, $\lim_{x \to \infty} \beta(x) \to 1$, and the continuity of $\beta(x)$ ensures the uniqueness of $x^*(q)$.

To establish a globally decreasing function, we augment the denominator with $x^2$ for $x < x^*(q)$ and equate them at the boundary. Redefine $\gamma$ as, $$
\begin{aligned}
\gamma(x; q) &\doteq \frac{\log\left(1 + x\right) - x}{1_{x \leq 0} x^2 + 1_{x > 0} \min\left(x^2, c^*(q) x^q\right)}, \\
\end{aligned}
$$ then for $x > -1$ and $\lambda \in \left[0, 1 + W_0\left(-e^{-2}\right)\right]$, $$
\begin{aligned}
-\gamma\left(\lambda x; q\right) &\leq -\gamma\left(-\lambda; q\right) & \left(x > -1\right) \\
&\leq 1, & \left(\lambda \in \left[0, 1 + W_0\left(-e^{-2}\right) \right]\right) \\
\lambda x - \log\left(1 + \lambda x\right) &\leq 1_{x \leq 0} \lambda^2 x^2 + 1_{x > 0} \min\left(\lambda^2 x^2, \lambda^q c^*(q) x^q\right) \\
&\leq \lambda^q \left( 1_{x \leq 0} x^2 + 1_{x > 0} \min\left(x^2, c^*(q) x^q\right) \right). & \left(\lambda \in \left[0, 1\right)\right) \\
\end{aligned}
$$
Here $W_0(z)$ is the principal branch of the Lambert W function and $1 + W_0\left(-e^{-2}\right) \approx 0.841406$.  To establish $\lim_{q \uparrow 2} c(q) = 1$, we expand $\left(\beta^{-1}(q)\right)^{2-q}$ for $q < 2$, $$
\begin{aligned}
c^*(q) = \left(\beta^{-1}(q)\right)^{2-q} &= 1 + \log\left(\frac{3}{2} (2 - q)\right) (2 - q) + O\left(\left(q - 2\right)^2\right). \\
\end{aligned}
$$
\end{proof}

\subsection{Proof of \cref{thm:qasym}}

\qasym*
\begin{proof}
Our object of interest is the boundary $$
\begin{aligned}
\mathcal{M}_{\alpha}\left(X_{\leq t}, \hat{X}_{\leq t}\right) &= 
\sup\left\{ s \in \mathbb{R} : \int_0^{\lambda_{\max}} \exp\left(\lambda s - v(\lambda)\right) dF(\lambda) \leq \frac{1}{\alpha} \right\},
\end{aligned}
$$ where $\alpha \in (0, 1)$, $$
\begin{aligned}
v(\lambda) &\doteq \sum_{s \leq t} \left(\lambda \left(X_s - \hat{X}_s\right) - \log\left(1 + \lambda \left(X_s - \hat{X}_s\right)\right)\right),
\end{aligned}
$$ and $F$ is assumed absolutely continuous with Lebesque measure, and with $\frac{dF}{d\lambda}(\lambda) = f(0) \lambda^{\frac{q}{2}-1} + O\left(\lambda^{q/2}\right)$ in a neighborhood of zero $[0, \zeta)$, with $f(0) > 0$.  From \cref{lemma:qgrowth} we can upper bound the variance, and hence upper bound the boundary, via $$
\begin{aligned}
\mathcal{M}_{\alpha}\left(X_{\leq t}, \hat{X}_{\leq t}\right) \leq \mathcal{M}^{(q)}_{\alpha}(v) &\doteq 
\sup\left\{ s \in \mathbb{R} : \underbrace{\int_0^{\min(\zeta, \lambda_{\max})} \exp\left(\lambda s - \lambda^q v\right) dF(\lambda)}_{\doteq m_q(s, v)} \leq \frac{1}{\alpha} \right\},
\end{aligned}
$$
where $\mathcal{M}^{(q)}_{\alpha}(v)$ will be evaluated at $$
\begin{aligned}
v &\doteq \sum_{s \leq t} 1_{X_s \leq \hat{X}_s} \left(X_s - \hat{X}_s\right)^2 + 1_{X_s > \hat{X}_s} \min\left(\left(X_s - \hat{X}_s\right)^2, c^*(q) \left(X_s - \hat{X}_s\right)^q\right),
\end{aligned}
$$ and we have further upper bounded the boundary by restricting to the absolutely continuous neighborhood.  Our proof steps are as follows:
\begin{enumerate}
    \item Sandwich $v^{1/q} = o\left(\mathcal{M}^{(q)}_{\alpha}(v)\right)$ via  \cref{lemma:lowersandwich} and $\mathcal{M}^{(q)}_{\alpha}(v) = o\left(v\right)$ via \cref{lemma:uppersandwich}.  This is intuitive given the $q$-LIL rate of $O\left(v^{1/q} \left(\log \log v\right)^{(q-1)/q}\right)$ from Theorem 5.1 of \citet{shao1997self}.
    \item Change variables $z \leftarrow \lambda^{2/q}$ and apply a modified Laplace method obtained by customizing the proof technique of \citet{fulks1951generalization}, realized as \cref{thm:customfulks}. The constraint on our choice of $F$ is a consequence of the change of variables formula, as we require non-zero finite density at the origin in the transformed space.  
    \item Invert the boundary: the sandwich lemmas, combined with continuity of the integral wrt $s$ and the intermediate value theorem, allows us to assert an exact boundary.
\end{enumerate}
Inverting the boundary, as $v \to \infty$, $$
\begin{aligned}
\frac{1}{\alpha} &= m\left( \mathcal{M}^{(q)}_{\alpha}(v) , v \right) \\
&\geq \frac{2}{q} f(0) \sqrt{\frac{2 \pi q}{4 (q - 1) v}} \exp\left(c(q) \left(\frac{\mathcal{M}^{(q)}_{\alpha}(v) }{v^{1/q}}\right)^{\frac{q}{q-1}}\right) \left( 1 + o(1) \right), \\
&\doteq f(0) a(q) \frac{1}{\sqrt{v}} \exp\left( \left(\frac{\mathcal{M}^{(q)}_{\alpha}(v)}{v^{1/q}}\right)^{\frac{q}{q-1}}\right) \left( 1 + o(1) \right), \\
\mathcal{M}^{(q)}_{\alpha}(v) &\leq  v^{1/q} \left( \log\left(\frac{\sqrt{v}}{\alpha f(0) a(q) \left(1 + o(1)\right)}\right) \right)^{\frac{q-1}{q}}, \\
\mathcal{M}_{\alpha}\left(X_{\leq t}, \hat{X}_{\leq t}\right) &\leq  v^{1/q} \left( \log\left(\frac{\sqrt{v}}{\alpha f(0) a(q) \left(1 + o(1)\right)}\right) \right)^{\frac{q-1}{q}}.
\end{aligned}
$$
Supporting propositions follow.
\end{proof}

\subsection{$v^{1/q} = o\left(\mathcal{M}^{(q)}_{\alpha}(v)\right)$}

First we restate Theorem 4.1 of \citet{olver1997asymptotics}.
\begin{manualtheorem}{Olver 4.1}[\citet{olver1997asymptotics} pg. 332] 
\label{thm:olver41}
Let 
$$
I(x) = \int_0^{\lambda_{\max}} \exp\left(-x p(\lambda) + x^{b/c} r(\lambda) \right) q(\lambda)\, d\lambda,
$$ and assume that 
\begin{enumerate}
    \item In the interval $(0, \lambda_{\max}]$, $p'(\lambda)$ is continuous and positive, and the real or complex functions $q(\lambda)$ and $r(\lambda)$ are continuous.
    \item As $\lambda \downarrow 0$, $$
    \begin{aligned}
    p(\lambda) &= p(0) + P \lambda^c + O\left(\lambda^{c_1}\right), & p'(\lambda) &= c P \lambda^{c-1} + O\left(\lambda^{c_1-1}\right), \\
    q(\lambda) &= Q \lambda^{a-1} + O\left(\lambda^{a_1-1}\right), & r(\lambda) &= R \lambda^b + O\left(\lambda^{b_1}\right)
    \end{aligned}
    $$
    where $$
    \begin{aligned}
    P > 0,\qquad c_1 > c > 0, \qquad a_1 > a > 0, \qquad c > b \geq 0, \qquad b_1 > b.
    \end{aligned}
    $$
\end{enumerate}
Then $$
\begin{aligned}
I(x) &= \frac{Q}{c} \faxen\left(\frac{b}{c}, \frac{a}{c}, \frac{R}{P^{b/c}}\right) \frac{e^{-x p(0)}}{(P x)^{a/c}} \left(1 + O\left(\frac{1}{x^{d/c}}\right) \right) & \left(x \to \infty \right)
\end{aligned}
$$
where $d = \min\left(a_1 - a, c_1 - c, b_1 - b\right)$ and
$$
\begin{aligned}
\faxen(\alpha, \beta, y) &\doteq \int_0^\infty \exp\left(-\tau + y \tau^\alpha\right) \tau^{\beta-1} d\tau & \left(0 \leq \mathrm{Re}(\alpha) < 1, \mathrm{Re}(\beta) > 0 \right)
\end{aligned}
$$ is the Fax\'en integral.
\end{manualtheorem}
\begin{lemma}
\label{lemma:lowersandwich}
$v^{1/q} = o\left(\mathcal{M}^{(q)}_{\alpha}(v)\right)$
\end{lemma}
\begin{proof}
For any $\kappa > 0$, to apply \cref{thm:olver41} to $m_q\left(\kappa v^{1/q}, v\right)$ establish the correspondences: $$
\begin{aligned}
x &\leftarrow v \\
p(\lambda) &\leftarrow \lambda^q & \left(P = 1, c = q, c_1 = q + 1\right), \\
r(\lambda) &\leftarrow \kappa \lambda & \left(R = \kappa, b = 1, b_1 = 2\right), \\
q(\lambda) &\leftarrow \frac{dF}{d\lambda}(\lambda) & \left(Q = f(0), a = \frac{q}{2}, a_1 = a + 1\right), \\
d &\leftarrow \min\left(1, (q+1) - q, 2-1\right) = 1,
\end{aligned}
$$ resulting in $$
\begin{aligned}
m_q\left(\kappa v^{1/q}, v\right) &= \frac{f(0)}{q} \faxen\left(\frac{1}{q}, \frac{1}{2}, \kappa\right) \frac{1}{v^{1/2}} \left(1 + O\left(\frac{1}{v^{1/q}}\right) \right) & \left(v \to \infty \right).
\end{aligned}
$$ 
Using known properties of the Fax\'en integral~\citep{kaminski1997asymptotics}, $$
\begin{aligned}
\faxen(a, b, 0) &= \Gamma(b), \\
\lim_{\kappa \to \infty} \faxen(a, b, \kappa) &= \sqrt{\frac{2 \pi}{1 - a}} \left(a \kappa\right)^{(2 b - 1) / (2 - 2 a)} \exp\left(\left(1 - a\right) \left(a^a \kappa\right)^{1/(1-a)}\right) + o\left(1\right),
\end{aligned}
$$ we find for sufficiently large $v$ that $m_q\left(\kappa v^{1/q}, v\right) \leq \frac{1}{\alpha}$.
\end{proof}

\subsection{$\mathcal{M}^{(q)}_{\alpha}(v) = o\left(v\right)$}

First we restate Theorem 2.1 of \citet{olver1997asymptotics}.
\begin{manualtheorem}{Olver 2.1}[\citet{olver1997asymptotics} pg. 326] 
\label{thm:olver21}
Let $\lambda_{\max}$ and $X$ be fixed positive numbers, and $$
I(x) = \int_0^{\lambda_{\max}} \exp\left(-x p(\lambda) + r(x, \lambda) \right) q(x, \lambda)\, d\lambda,
$$ 
\begin{enumerate}
    \item $p'(\lambda)$ is continuous and positive in $(0, \lambda_{\max}]$ with $$
\begin{aligned} 
p(\lambda) &= p(0) + P \lambda^b + O\left(\lambda^{b_1}\right), \\
p'(\lambda) &= b P \lambda^{b-1} + O\left(\lambda^{b_1-1}\right),
\end{aligned}
$$ with $P > 0$ and $b_1 > b > 0$.
    \item For all $x \in [X, \infty)$, $r(x, \lambda)$ and $q(x, \lambda)$ are continuous for $\lambda \in (0, \lambda_{\max}]$.  Moreover, $$
\begin{aligned} 
\left| r(x, \lambda) \right| \leq R x^c \lambda^d,&\qquad  \left| q(x, \lambda) - Q \lambda^{a-1} \right| \leq Q_1 x^g \lambda^{a_1 - 1}, \\
\end{aligned}
$$ where $R$, $c$, $d$, $Q$, $a$, $Q_1$, $g$, and $a_1$ are independent of $x$ and $\lambda$, and 
$$
\begin{aligned}
d \geq 0,&\qquad a > 0, &\qquad a_1 > 0, &\qquad c < \min(1, d/b), &\qquad g < (a_1 - a) / b.
\end{aligned}
$$
\end{enumerate}
Then $$
\begin{aligned}
I(x) &= \frac{Q}{b} \Gamma\left(\frac{a}{b}\right) \frac{e^{-x p(0)}}{(P x)^{a/b}} \left(1 + O\left(\frac{1}{x^{h/b}}\right)\right) & \left(x \to \infty\right)
\end{aligned}
$$ where $h = \min(b_1 - b, d - b c, a_1 - a - b g)$.

Furthermore if $r(x, \lambda)$ is real and non-positive then the conditions relax to $c < d/b$.
\end{manualtheorem}

\begin{lemma}
\label{lemma:uppersandwich}
$\mathcal{M}^{(q)}_{\alpha}(v) = o\left(v\right)$
\end{lemma}
\begin{proof}
Consider the expected value of the inverse wealth, $$
\begin{aligned}
m^{(-1)}_q(s, v) &\doteq  \int_0^{\lambda_{\max}} \exp\left(-\lambda s + \lambda^q v\right) dF(\lambda).
\end{aligned}
$$ For any $\kappa > 0$, to apply \cref{thm:olver21} to $m^{(-1)}_q\left(\kappa v, v\right)$, establish the following correspondences: $$
\begin{aligned}
x &\leftarrow v, \\
p(\lambda) &\leftarrow \kappa \lambda - \lambda^q, & \left(P = \kappa, b = 1, b_1 = q\right) \\
r(x, \lambda) &\leftarrow 0, & \left(R = 0, c = 0, d = 1\right) \\
q(x, \lambda) &\leftarrow \frac{dF}{d\lambda}(\lambda), & \left(Q = f(0), a = \frac{q}{2}, g = 0, a_1 = 2, Q_1 = f'(0)\right) \\
h &\leftarrow \min\left(q - 1, 1 - 1 \cdot 0, 2 - \frac{q}{2} - 1 \cdot 0\right) = q - 1,
\end{aligned}
$$ where $c < \min(1, d/b) = 1$ holds. This results in $$
\begin{aligned}
m^{(-1)}_q\left(v, v\right) &= f(0) \Gamma\left(\frac{q}{2}\right) \frac{1}{\kappa^{q/2} v^{q/2}} \left(1 + O\left(\frac{1}{v^{q-1}}\right)\right) & \left(v \to \infty\right).
\end{aligned}
$$ Thus for sufficiently large $v$, $m^{(-1)}_q\left(\kappa v, v\right) \leq \alpha$.  From Jensen's inequality, $$
\frac{1}{m_q\left(\kappa v, v\right)} \leq m^{(-1)}_q\left(\kappa v, v\right) \leq \alpha \implies m_q\left(\kappa v, v\right) \geq \frac{1}{\alpha}.
$$
\end{proof}

\subsection{Customized \citet{fulks1951generalization}}

Ideally we could reuse Theorem 4 of \citet{fulks1951generalization} directly, but this was not possible.  Instead we follow the same line of argument, changing one critical step to accommodate our scenario.

\begin{theorem}[Customization of \citet{fulks1951generalization} Theorem 4]
\label{thm:customfulks}
For $1 < q < 2$ and $\eta > 0$, let 
$$
\begin{aligned}
m(s, v) &\doteq \int_0^\eta \exp\left(\lambda s - \lambda^q v\right) dF(\lambda),
\end{aligned}
$$ where $F$ is absolutely continuous with Lebesque measure, with $\frac{dF}{d\lambda}(\lambda) = f(0) \lambda^{\frac{q}{2}-1} + O\left(\lambda^{\frac{q}{2}}\right)$ and $f(0) > 0$.  For $s, v \to \infty$, with $s = o(v)$ and $v^{1/q} = o(s)$, we have $$
\begin{aligned}
m(s, v) &= \frac{2}{q} f(0) \sqrt{\frac{2 \pi q}{4 (q - 1) v}} \exp\left(c(q) \left(\frac{s}{v^{1/q}}\right)^{\frac{q}{q-1}}\right) \left[ 1 + o(1), \frac{1}{2 \sqrt{1 - \log(2)}} + \frac{q}{4(q - 1)} + o(1) \right],
\end{aligned}
$$ where $$
\begin{aligned}
c(q) &\doteq \left(\frac{1}{q}^{\frac{1}{q-1}} - \frac{1}{q}^{\frac{q}{q-1}}\right) > 0.
\end{aligned}
$$
\end{theorem}
\begin{proof}
First we change the integration variable,
$$
\begin{aligned}
m(s, v) &= \int_0^{b} \exp\left(t^{2/q} s - t^2 v\right) dG(t) & \left(t \leftarrow \lambda^{q/2}\right), \\
\end{aligned}
$$ where $b = \eta^{q/2}$ and $dG(t) = \frac{2}{q} t^{\frac{2}{q} - 1} dF(t^{2/q})$.  Note $g(t) = \frac{2}{q} f(0) + O\left(t\right)$.  Next we define the value $\tau$ as\footnote{Contrasting with \citet{fulks1951generalization}, we exactly compute $\tau$ here, which allows us to proceed.} $$
\begin{aligned}
0 &= \frac{\partial}{\partial t}\left(t^{2/q} s - t^2 v\right)\biggr|_{t=\tau} \implies \tau = \left(\frac{1}{q}\right)^{\frac{q}{2 (q - 1)}} \left(\frac{s}{v}\right)^{\frac{q}{2 (q - 1)}}.
\end{aligned}
$$ Since $s = o(v)$ we have $\tau = o(1)$.  We also have  
$$
\begin{aligned}
\tau v^{1/2} &= \left(\frac{1}{q}\right)^{\frac{q}{2 (q - 1)}}\left(\frac{s}{v^{1/q}}\right)^{\frac{q}{2 (q - 1)}} \to \infty, \\
\left(t^{2/q} s - t^2 v\right)\biggr|_{t=\tau} &= \left(\frac{1}{q}^{\frac{1}{q-1}} - \frac{1}{q}^{\frac{q}{q-1}}\right) \left(\frac{s}{v^{1/q}}\right)^{\frac{q}{q-1}} \doteq c(q) \left(\frac{s}{v^{1/q}}\right)^{\frac{q}{q-1}} \to \infty,
\end{aligned}
$$ since $v^{1/q} = o(s)$.  Choose $s, v$ large enough such that $\tau < b$.  We start by assuming $dG(t)$ is constant and relax this assumption later.  Consider the centered integrand, $$
\begin{aligned}
I_1 &= \int_0^b \exp\left(\left(t^{2/q} - \tau^{2/q}\right) s - \left(t^2 - \tau^2\right) v\right) dt \\
&= \int_0^{\tau - \delta_1} \ldots + \int_{\tau - \delta_1}^{b} \ldots \doteq I_2 + I_3.
\end{aligned}
$$ For $t > \left(\frac{2-q}{q}\right)^{\frac{q}{2 (q - 1)}} \tau$ the exponent is strictly concave.  Noting $$
\sup_{q \in (1, 2]} \left(\frac{2-q}{q}\right)^{\frac{q}{2 (q - 1)}} = e^{-1},
$$ we take $\delta_1 = \frac{1}{2} \tau$, allowing us to crudely bound $I_2 \leq o(1)$ via the exponent being non-positive.  For $I_3$ we Taylor expand around $t = \tau$, $$
\begin{aligned}
I_3 &= \int_{\tau - \delta_1}^{\tau + b} \exp\left( -\frac{1}{2} v \left( 2 - \frac{s}{v} h''(t_1) \right) \left(t - \tau\right)^2 \right) dt \\
&= \int_{\tau - \delta_1}^\tau \ldots + \int_\tau^{\tau + b} \ldots \doteq I^{(-)}_3 + I^{(+)}_3,
\end{aligned}
$$ where $t_1$ is between $t$ and $\tau$, and $h(t) \doteq t^{2/q}$.  For $I^{(-)}_3$ we have the sandwich $$
t_1 \in \left[\tau - \delta_1, \tau\right] \implies 2 - \frac{s}{v} h''(t_1) \in \left[2 - 2^{3-\frac{2}{q}} \left(\frac{2 - q}{q}\right), 4 - \frac{4}{q} \right] \doteq \left[h_0, h_1\right]
$$ where $h_0 > 0$, so we substitute $(t - \tau) = -y / v^{1/2}$ to get $$
\begin{aligned}
\frac{1}{v^{1/2}} \int_0^{\frac{1}{2} \tau v^{1/2}} \exp\left(-\frac{1}{2} h_0 y^2\right) \leq I^{(-)}_3 &\leq \frac{1}{v^{1/2}} \int_0^{\frac{1}{2} \tau v^{1/2}} \exp\left(-\frac{1}{2} h_1 y^2\right).
\end{aligned}
$$ The limit of integration diverges and we get $$
\begin{aligned}
I^{(-)}_3 &\in \left[ \sqrt{\frac{\pi}{2 h_1 v}}, \sqrt{\frac{\pi}{2 h_0 v}} \right] \in \sqrt{\frac{\pi}{2 h_1 v}} \left[ 1, \frac{1}{\sqrt{1 - \log(2)}} \right],
\end{aligned}
$$ noting $\left(1 - \log(2)\right) h_1 \leq h_0$. For $I^{(+)}_3$ we have the sandwich $$
t_1 \in \left[\tau, b\right] \implies 2 - \frac{s}{v} h''(t_1) \in \left[4 - \frac{4}{q}, 2 + \frac{2 (q - 2) b^{2/q-2}}{q^2} \left(\frac{s}{v}\right) \right] \doteq \left[h_1, h_2\right]
$$ noting $h_2 = 2 + o(1)$, so we substitute $(t - \tau) = y / v^{1/2}$ to get $$
\begin{aligned}
\frac{1}{v^{1/2}} \int_0^{(b - \tau) v^{1/2}} \exp\left(-\frac{1}{2} h_1 y^2\right) \leq I^{(+)}_3 &\leq \frac{1}{v^{1/2}} \int_0^{(b - \tau) v^{1/2}} \exp\left(-\frac{1}{2} h_2 y^2\right).
\end{aligned}
$$ The limit of integration diverges and we get $$
\begin{aligned}
I^{(+)}_3 &\in \left[ \sqrt{\frac{\pi}{2 h_1 v}}, \sqrt{\frac{\pi}{2 h_2 v}} \right] \in \sqrt{\frac{\pi}{2 h_1 v}} \left[ 1, \frac{q}{2 (q - 1)} + o(1) \right],
\end{aligned}
$$ Combining yields $$
\begin{aligned}
I_1 &\in \sqrt{\frac{2 \pi q}{4 (q - 1) v}} \left[ 1, \frac{1}{2 \sqrt{1 - \log(2)}} + \frac{q}{4(q - 1)} + o(1) \right] + o(1).
\end{aligned}
$$
To show for our actual $dG(z)$ it suffices to bound $$
\begin{aligned}
I_6 &\doteq \int_0^b \left(g(t) - g(0)\right) \exp\left(\left(t^{2/q} - \tau^{2/q}\right) s - \left(t^2 - \tau^2\right) v\right) dt \\
&= \int_0^{c} \ldots + \int_{c}^b \ldots \doteq I^{(-)}_6 + I^{(+)}_6,
\end{aligned}
$$ where $c > \tau$.  From the strong concavity result above, we choose $c = O(\sqrt{\tau})$ to ensure $I^{(+)}_6 = \exp\left(O\left(-v \tau\right)\right) = o(1)$.  Meanwhile for $I^{(-)}_6$ we have $g(t) - g(0)$ upper bounded by $O\left(\sqrt{\tau}\right) = o(1)$.

\end{proof}

\section{Computational Considerations}
\label{app:computationalconsiderations}

\subsection{Proof of \cref{thm:tildegprox}}

\tildegprox*
\begin{proof}
For $x \in (-1, 0]$, $$
\begin{aligned}
g(\lambda, x) &= \lambda x - \log\left(1 + \lambda x\right) \\
&\leq \frac{\lambda^2 x^2}{2 (1 + \lambda x)} & \left( \forall z \in (-1, 0]: \log\left(1 + z\right) \geq \frac{z}{2} \frac{2 + z}{1 + z} \right) \\
&\leq \frac{\lambda^2}{2 (1 - \lambda)} x^2 & (x > -1), \\
&= g(\lambda, x) + \left(\frac{x^2}{2} + \frac{x^3}{3}\right) \lambda^3 + O(\lambda^4).
\end{aligned}
$$ There is a single sufficient statistic for $x \in (-1, 0]$, $\sum_{s \leq t} 1_{X_s - \hat{X}_s \leq 0} \left(X_s - \hat{X}_s\right)^2$.

For $x \in (0, 1]$, $$
\begin{aligned}
g(\lambda, x) &= \lambda x - \log\left(1 + \lambda x\right) \\
&\leq \frac{\lambda^2}{2 + x \lambda} x^2 & \left( \forall z > 0: \log\left(1 + z\right) \geq \frac{2 z}{2 + z} \right) \\
&\leq \frac{\lambda^2}{2} x^2 & \left(x > 0\right) \\
&= g(\lambda, x) + \frac{x^3}{3} \lambda^3 + O(\lambda^4).
\end{aligned}
$$ There is a single sufficient statistic for $x \in (0, 1]$, $\sum_{s \leq t} 1_{X_s - \hat{X}_s \in (0, 1]} \left(X_s - \hat{X}_s\right)^2$.

For $x > 1$: we have $x \in [x_1, x_2) = [k^n, k^{n+1})$.  $g(\lambda, x)$ is strongly convex in this region with strong convexity parameter $m \doteq \frac{\lambda^2}{\left(1 + \lambda x_2\right)^2}$, therefore $$
\begin{aligned}
g(\lambda, x) &= \lambda x - \log\left(1 + \lambda x\right) \\
&\leq \alpha g(\lambda, x_1) + \left(1 - \alpha\right) g(\lambda, x_2) - \frac{1}{2} m \alpha \left(1 - \alpha\right) (k - 1)^2 x_1^2 \\
&= g(\lambda, x) + \frac{(k - 1)^3 \alpha (1 - \alpha^2)}{3 \left(k + \alpha - k \alpha\right)^3} x^3 \lambda^3 + O(\lambda^4) \\
&\leq g(\lambda, x) + \max_{\alpha \in [0, 1]} \frac{(k - 1)^3 \alpha (1 - \alpha^2)}{3 \left(k + \alpha - k \alpha\right)^3} x^3 \lambda^3 + O(\lambda^4) \\
&= g(\lambda, x) + \frac{2}{9 \sqrt{3}} x^3 (k - 1)^3 \lambda^3 + O\left((k-1)^4 \lambda^3\right) + O(\lambda^4).
\end{aligned}
$$ For each $s$, there are three sufficient statistics to accumulate: $$
\begin{aligned} 
n_s &\doteq \left\lfloor\log_k\left(X_s - \hat{X}_s\right)\right\rfloor, \\
\alpha_s &\doteq \frac{k^{n_s+1} - \left(X_s - \hat{X}_s\right)}{k^{n_s+1} - k_{n_s}}, \\
z_{n_s} &\leftarrow z_{n_s} + \alpha_s, \\
z_{n_{s+1}} &\leftarrow z_{n_{s+1}} + \left(1 - \alpha_s\right), \\
y_{n_s} &\leftarrow \frac{1}{2} \alpha_s (1 - \alpha_s), 
\end{aligned}
$$ from which for any $\lambda$ we can compute $$
\begin{aligned}
\sum_{s \leq t} 1_{X_s - \hat{X}_s > 1} \tilde{g}\left(\lambda, X_s - \hat{X}_s\right) &= \sum_n z_n g(\lambda, k^n) - \sum_n y_n \frac{\lambda^2}{(1 + \lambda k^{n+1})^2} (k - 1)^2 k^{2n}.
\end{aligned}
$$ 
The statement about computational cost follows from upper bounding the number of non-zero sufficient statistics.
\end{proof}

\subsection{Proof of \cref{thm:ddrm}}

\ddrm*
\begin{proof}

We start by replacing $g \leftarrow \tilde{g}$ in the $q$-specific construction from \cref{thm:qasym}.  Then we apply the mixture strategy from \citet{howard2021time}:  for some $\xi > 1$ we discretize the mixture distribution $$
\begin{aligned}
\mathcal{DM}_{q, \alpha}\left(X_{s \leq t}, \hat{X}_{s \leq t}\right) &= \sup\left\{ s \in \mathbb{R} : \sum_{j=0}^\infty y_j \exp\left(\lambda_j s - \sum_{s \leq t} \tilde{g}\left(\lambda_j, X_s - \hat{X}_s\right) \right) \leq \frac{1}{\alpha} \right\}, \\
\lambda_j &= \frac{\lambda_{\max}}{\xi^{j+1/2}}, \\
y_j &= \frac{\lambda_{\max} (\xi - 1) f(\lambda_j \sqrt{\xi})}{\xi^{j + 1}}\biggr|_{f(z) = \frac{q z^{q/2-1}}{2 \lambda_{\max}^{q/2}}} = \frac{q}{2} \left(\frac{\xi - 1}{\xi^{1 + \frac{j q}{2}}}\right).
\end{aligned}
$$
Embedding the previous into the robust construction of \cref{cor:qadapt} yields $$
\begin{aligned}
\mathcal{DM}_{\alpha}\left(X_{s \leq t}, \hat{X}_{s \leq t}\right) &= \sup\left\{ s \in \mathbb{R} : \sum_{j=0}^\infty \left(\sum_{k=0}^\infty w_k y_{j,k}\right) \exp\left(\lambda_j s - \sum_{s \leq t} \tilde{g}\left(\lambda_j, X_s - \hat{X}_s\right) \right) \leq \frac{1}{\alpha} \right\}, \\
\lambda_j &= \frac{\lambda_{\max}}{\xi^{j+1/2}}, \\
y_{j,k} &= \frac{q(k)}{2} \left(\frac{\xi - 1}{\xi^{1 + \frac{j q(k)}{2}}}\right) = \frac{q(k)}{2} \left(\frac{\xi - 1}{\xi^{1 + j}}\right) \xi^{j \left(1 - \frac{q(k)}{2}\right)}.
\end{aligned}
$$ For some $r > 1$ the choice $w_k = \zeta(r)^{-1} (k + 1)^{-r}$ yields $$
\begin{aligned}
\sum_{k=0}^\infty w_k y_{j,k} &\geq \frac{1}{2} \left(\frac{\xi - 1}{\xi^{1 + j}}\right) \frac{1}{\zeta(r)} \sum_{k=0}^\infty (k + 1)^{-r} \left(1 + \eta^k\right) = \frac{1}{2} \left(\frac{\xi - 1}{\xi^{1 + j}}\right) \left(1 + \frac{1 + \polylog_r(\eta)}{\eta \zeta(r)}\right).
\end{aligned}
$$ 
\end{proof}

\end{document}